\documentclass[submission,copyright]{eptcs}
\providecommand{\event}{TARK 2015} % Name of the event you are submitting to
\usepackage{breakurl}             % Not needed if you use pdflatex only.

\usepackage{amsmath}
\usepackage{amsthm}
\usepackage{amssymb}
\usepackage{stmaryrd}
\usepackage[all]{xy}
\usepackage{tikz}

\usepackage{appendix}

\usetikzlibrary{calc}
\newenvironment{e-p-sys}[2][]
{\begin{tikzpicture}
[dot/.style={circle,inner sep=1pt},line/.style={thick,gray}]
\foreach \x in {1,...,#2} \foreach \y in {1,...,#2} {
\ifnum\y=\x 
	\node(\x-\y) at ($(\x*-360/#2+360/#2+90:.08*#2^1.6)+(\y*-360/#2+360/#2+90:.05*#2^1.1)$) [dot,fill=black] {};
\else
	\node(\x-\y) at ($(\x*-360/#2+360/#2+90:.08*#2^1.6)+(\y*-360/#2+360/#2+90:.05*#2^1.1)$) [dot,fill=gray] {};
\fi
}
\node(label) {#1};}
{\end{tikzpicture}}

\newtheorem{thm}{Theorem}
\newtheorem{cnj}{Conjecture}
\newtheorem{definition}{Definition}

\newcommand{\KUI}{$\mathit{KU}$-Intro\-spection}
\newcommand{\AUI}{$\mathit{AU}$-Intro\-spection}
\newcommand{\CK}{\mathit{CK}}
\newcommand{\ck}{\mathit{ck}}

\title{Standard State Space Models of Unawareness\\(Extended Abstract)\footnote{A draft of the full version of this paper is available at \url{http://users.ox.ac.uk/\~hert2388/}}}
\author{Peter Fritz
\institute{Department of Philosophy,\\Classics, History of Arts and Ideas\\
University of Oslo, Norway}
\email{peter.fritz@ifikk.uio.no}
\and
Harvey Lederman
\institute{Department of Philosophy\\
New York University, USA}
\email{hsl306@nyu.edu}
}
\def\titlerunning{Standard State Space Models of Unawareness}
\def\authorrunning{Peter Fritz \& Harvey Lederman}
\begin{document}
\maketitle

\begin{abstract}
The impossibility theorem of Dekel, Lipman and Rustichini \cite{dekellipmanrustichini1998sssmpu}
has been thought to demonstrate that standard state-space models cannot be used to represent unawareness. We first show that Dekel, Lipman and Rustichini do not establish this claim. We then distinguish three notions of awareness, and argue that although one of them may not be adequately modeled using standard state spaces, there is no reason to think that standard state spaces cannot provide models of the other two notions. In fact, standard space models of these forms of awareness are attractively simple. They allow us to prove completeness and decidability results with ease, to carry over standard techniques from decision theory, and to add propositional quantifiers straightforwardly.
\end{abstract}

%\keywords{awareness, standard state space models, epistemic logic}

\section{Introduction}

Dekel, Lipman and Rustichini \cite{dekellipmanrustichini1998sssmpu}, hereafter, ``DLR'', claim to show that ``standard state space models preclude unawareness''. Their claim has achieved the status of orthodoxy.\footnote{See, e.g., \cite[p.~2]{schipper2014u}, \cite{schipperforthcominga}, \cite[p.~78]{heifetzetal2006iu}, \cite[p.~305]{heifetzetal2008acmfiu}, \cite[p.~101]{heifetzetal2013ubast}, \cite[p.~220]{meierschipper2014bgwuaup}, \cite[p.~2790]{karniviero2013rb}, \cite[pp.~977--978]{li2009iswu}, \cite[p.~2454]{heinsalu2012eofiswuttloa}, \cite[p.~257]{heinsalu2014utswu}, \cite[p.~42]{galanis2013uot}, \cite[p.~516]{sillari2008qloaaipw}, and \cite{sillari2006moa}.}
The first task of this paper is to clear the way for standard state space models of unawareness by showing that the formal result DLR present does not establish their headline conclusion. DLR informally motivate certain axioms concerning unawareness, but in their formal impossibility result, they rely on the claim that these axioms hold at all states in the model. As section~2 argues, the assumption that axioms hold at all states of the model is unwarranted; in fact, DLR themselves reject it. While DLR's formal results are valid, they are not sufficiently general to rule out standard state space models of unawareness. As we show, the impossibility results do not hold if one assumes only that DLR's explicit assumptions about unawareness hold at some ``real'' states, as opposed to at all states. Even strengthening those explicit assumptions considerably does not reinstate the results.

But this does not yet vindicate standard state space models of unawareness. Section~3 presents a novel impossibility result which uses widely shared assumptions about awareness. The new impossibility theorem relies on the assumption that an agent who is aware of a conjunction is aware of its conjuncts. If awareness satisfies this assumption, then standard state space models do in fact preclude unawareness. We then distinguish three notions of awareness, and suggest that two important ones do not satisfy this assumption, leaving open the possibility that they could be adequately modeled by standard state space models.

The remainder of the paper continues in a more positive vein. We describe a simple class of standard state space models which represent key features of awareness. In section~4, we establish completeness and decidability for the logic of these models. We also show that adding propositional quantifiers, a topic which has presented major difficulties for existing approaches to awareness, is straightforward in our standard state space models. In section~5, we present one way of implementing a choice-based approach to decision theory within these models, and show how non-trivial unawareness is consistent with speculative trade.  Section~6 concludes.

\section{DLR's Triviality Result}

\subsection{Standard State Space Models}

Standard state space models for the knowledge and awareness of a single agent can be understood as certain tuples $\langle\Omega,k,a\rangle$. $\Omega$ is required to be a set, called the set of \emph{states}, from which a set of events is derived by taking an \emph{event} to be a set of states. $k$ and $a$ are functions on events, which represent the agent's knowledge and awareness, respectively: $k$ maps each event $E$ to the event $k(E)$ of the agent knowing $E$; $a$ similarly takes each $E$ to the event $a(E)$ of the agent being aware of $E$.

Such models are straightforwardly used to interpret a formal language in which one can talk about knowledge and awareness. Let $L$ be such a language built up from proposition letters $p,q,\dots$, using a unary negation operator $\neg$, a binary conjunction connective $\wedge$ and two unary operators $K$ and $A$, respectively ascribing knowledge and awareness to the agent. Formulas of this language are interpreted relative to a model $M=\langle\Omega,k,a\rangle$ and a valuation function $v$ which maps each proposition letter $p$ to the event $v(p)$. The interpretation uses a function $\llbracket\cdot\rrbracket_{M,v}$ which maps each formula $\varphi$ of $L$ to the event expressed by $\varphi$ in $M$, which can be understood as the set of states in which $\varphi$ is true in $M$. To state the constraints on such a function let $-E=\Omega\backslash E$.
\begin{itemize}
\item[] $\llbracket p\rrbracket_{M,v}=v(p)$
\item[] $\llbracket\neg\varphi\rrbracket_{M,v}=-\llbracket\varphi\rrbracket_{M,v}$
\item[] $\llbracket\varphi\wedge\psi\rrbracket_{M,v}=\llbracket\varphi\rrbracket_{M,v}\cap\llbracket\psi\rrbracket_{M,v}$
\item[] $\llbracket K\varphi\rrbracket_{M,v}=k(\llbracket\varphi\rrbracket_{M,v})$
\item[] $\llbracket A\varphi\rrbracket_{M,v}=a(\llbracket\varphi\rrbracket_{M,v})$
\end{itemize}
The agent being unaware of something can of course be understood as it not being the case that she is aware of it. We therefore syntactically use $U\varphi$ as an abbreviation for $\neg A\varphi$. Similarly, we introduce the other connectives of classical propositional logic as abbreviations, using $\vee$ for disjunction, $\to$ for material implication, $\leftrightarrow$ for bi-implication, and $\top$ and $\bot$ for an arbitrary tautology and contradiction, respectively. On the semantic side, we adopt the convention of writing $fg$ for the composition of functions $f$ and $g$, which allows us two write, e.g., $k-a(E)$ instead of $k(-(a(E)))$.

In order to express general constraints on these models, we say that a formula $\varphi$ is \emph{valid on $M$} if $\llbracket\varphi\rrbracket_{M,v}=\Omega$ for each valuation function $v$; this can be understood as requiring $\varphi$ to be true in every state of $M$ according to every valuation function. In order to limit this constraint to a particular state $\omega\in\Omega$, we say that $\varphi$ is \emph{valid in $\omega$} if $\omega\in\llbracket\varphi\rrbracket_{M,v}$ for each valuation function $v$.

These models count as ``standard'' in the sense of DLR. First, the events expressed by $A\varphi$ and $K\varphi$ are each a function of the event expressed by $\varphi$. (DLR call this ``event-sufficiency''.) Second, negation is interpreted as set-comple\-ment and conjunction as intersection, so that all tautologies of classical propositional logic, such as $p\vee\neg p$, are interpreted as the set of all states in every model. (DLR call this assumption ``real states''.)

\subsection{DLR on Standard State Space Models}

DLR introduce three constraints on awareness, which can be stated using the following three axioms:
\begin{itemize}
\item[] Plausibility: $Up\to(\neg Kp\wedge\neg K\neg Kp)$
\item[] \KUI{}: $\neg KUp$
\item[] \AUI{}: $Up\to UUp$
\end{itemize}
Their constraints on knowledge can be stated using the following three axioms:
\begin{itemize}
\item[] Necessitation: $K\top$
\item[] Monotonicity: $K(p\wedge q)\to(Kp \wedge Kq)$
\item[] Weak Necessitation: $Kp\to K\top$
\end{itemize}
Their main results are then:

\begin{thm}[DLR]
Let $M=\langle\Omega,k,a\rangle$ be a model on which Plausibility, \KUI{} and \AUI{} are valid.
\begin{description}
\item[1(i)] If Necessitation is valid on $M$, then $Ap$ is valid on $M$.
\item[1(ii)] If Monotonicity is valid on $M$, then $Kp\to Aq$ is valid on $M$.
\item[2] If Weak Necessitation is valid on $M$, then $Kp\to Aq$ and $Ap\leftrightarrow Aq$ are valid on $M$.
\end{description}
\end{thm}

Our presentation of DLR's result differs in superficial respects from their original presentation. DLR do not present their constraints in terms of the validity of certain axioms. Thus, for example, instead of requiring \KUI{} to be valid on $M$, they require $k-a(E)=\emptyset$ for all events $E$. However, it is a routine exercise to show that this condition is equivalent to the validity of our corresponding axiom. The same point holds for the other axioms. In short, our later models will not be escaping their triviality result by a sleight of hand which depends on this presentation.

One reason for the variant presentation is that it will facilitate the later exposition. It also serves to demonstrate that standard state space models as discussed here are equivalent to what are now commonly known as neighborhood or Scott-Montague frames (see \cite{scott1970aoml} and \cite{montague1970ug}). It is well known that given certain restrictions on the function interpreting knowledge, this function can be turned into a binary relation among states along the lines of those used by \cite{kripke1963saomli} and \cite{HintikkaKB}. This representation as a binary relation is, in turn, formally interchangeable with the ``possibility correspondences''  introduced by \cite{AumannAgreeing} (see also \cite{AumannKnowledge}) and used throughout economic theory (see, e.g. \cite{FaginKnowledge}).

\subsection{Two Kinds of States}

In response to their triviality results, DLR suggest distinguishing informally between ``real'' states and ``subjective'' states. As we understand it, this distinction can be explained as follows. An epistemic model makes predictions about how an agent or group of agents will or would behave in particular situations. The model makes predictions about these situations by including states which represent them. The real states in a model are the states which represent situations the model is intended to describe. The model predicts an agent will behave a certain way in a particular situation just in case the agent behaves that way in the real state which represents that situation. The predictions of a single model are given by what holds at all its real states; the behavioral theory of a class of models is given by what holds at all real states in all its models.

A state in an epistemic model is subjective if it figures in the specification of what the agent knows or is aware of at some real state. According to this way of understanding real and subjective states, states may be both real and subjective. Suppose we wish to represent an agent who knows that a particular coin will be flipped, but who will not learn the outcome of this coin flip. If our model is intended to make predictions no matter how the coin lands, the subjective states needed to specify the agent's knowledge (heads and ignorant, tails and ignorant) will be exactly the real states; every state will be both real and subjective. But as DLR recognize, there is no reason to require all states to play both roles. In the earlier example, if we only wanted to make predictions about the situation in which the coin comes up heads, we would not count one of the states (tails and ignorant) as a real state; even so, to represent the agent's ignorance given heads, a state where the coin comes up tails would still have to be included as a subjective state. The point can also be illustrated with a less artificial example. Consider an analyst who wishes to model the interactions of agents who are rational, but who do not believe each other to be rational. To represent the beliefs of these agents, the analyst must include subjective states in which the agents are irrational. But although she includes these subjective states, the analyst has no intention of eliminating the claim that the agents are rational from the predictions of her theory. Rather, it is understood informally that these subjective states are not real; they do not represent situations the analyst aims to describe. 

Put in the terms of our presentation, DLR's proposal is to allow subjective states in which the law of excluded middle, $p\vee\neg p$, may not be true. DLR have no intention of eliminating the law of excluded middle from the predictions of their theory. Rather, they introduce these subjective states to specify the agent's knowledge and unawareness at real states where classical logic holds. The theory of DLR's models is given by what holds at these classical real states, not by what holds at all states whatsoever. Still, since classical tautologies may fail in DLR's subjective states, their models violate the ``real states'' assumption, and so are not standard state space models.

But once we acknowledge that some states may not be intended to represent situations the analyst wishes to describe, a question arises: Why should one require DLR's three axioms on unawareness to be valid in \emph{all} states? DLR do not argue for this assumption. At best, their arguments in favor of their axioms only motivate imposing these axioms at real states. These arguments provide no motivation for imposing the axioms on subjective states that are not real, since these states are merely included in the model to specify the agent's knowledge and unawareness at real states.

There is a general methodological principle in epistemic modeling that \emph{axioms} are to be imposed at all states. But in the literature on awareness, following DLR, this methodological principle has long been abandoned. DLR's own non-standard models violate this requirement, as do the current leading proposals for representing awareness, for example that of \cite{heifetzetal2006iu}. The subjective states in DLR's models include states in which logical axioms, including the law of excluded middle, do not hold. In DLR's models, as in the ones we will propose, even logical axioms are allowed to fail at some states. The only difference between our proposal and theirs concerns which axioms are allowed to fail. DLR preserve their own axioms at all states, and move to non-standard models in which classical propositional logic may fail at subjective states. We will preserve classical propositional logic at all states, and work with standard models in which DLR's axioms may fail at subjective states.

Still, one might ask: How should we \emph{understand} a state where DLR's axioms are false? DLR interpret their own subjective states as ``descriptions of possibilities as perceived by the agent'' (p.~171). This interpretation does not seem appropriate for our models, in which DLR's axioms may fail at subjective states. But such metaphorical interpretations of these states are unnecessary. Subjective states where the axioms of awareness are invalid are simply to be understood in terms of the agent's knowledge and awareness at real states where the axioms are valid. 

In fact, we can give a direct argument for not imposing DLR's axioms at all states, and in particular, for including states where \KUI{} is invalid. First, by \AUI{}, if an agent is unaware of $p$, then she must be unaware of being unaware of $p$. But then, by Plausibility, the agent does not know that she does not know that she is unaware of $p$. (Essentially, DLR already give this argument on p.~169.) In epistemic models, we generally represent an agent's not knowing $q$ by including a state in which $q$ is false. So, to allow for real states in which the agent does not know that she does not know that she is unaware of $p$, we must include subjective states in which the agent knows that she is unaware of $p$, and so violates \KUI{}.%If we want to have real states representing situations which conform to DLR's axioms, and in which an agent exhibits unawareness, it is natural to include subjective states in which DLR's axioms fail.
\footnote{This argument differs from DLR's main proof of triviality, since it only assumes that the axioms hold at the real state.}
 
To sum up: after rejecting standard state space models, DLR propose that we should use models in which the laws of logic fail at subjective states. They implement this proposal by countenancing states where propositional logic fails, so that their models are non-standard. But if we allow models in which the law of excluded middle may fail at subjective states, we must also consider models in which other axioms, including DLR's, may fail at subjective states. DLR's formal results only apply to standard state space models in which their axioms are imposed at all states; the results do not concern standard state space models in which the axioms are imposed only at real states. As a consequence, these formal results cannot provide a basis for the conclusion DLR draw: that standard state space models preclude unawareness.

\subsection{Non-Triviality}

We have not argued against the validity of DLR's three axioms in real states -- states representing the situations to be modeled. In our setting, we can formalize the distinction between real states and subjective states which are not real, by only assuming the validity of the axioms in some subset of the states in a model. We can then ask: is assuming the validity of the three axioms in such distinguished real states enough to lead to triviality? The following example shows that it is not:

\begin{thm}\label{thm-counterexample}
There is a model $M=\langle\Omega,k,a\rangle$, state $\alpha\in\Omega$ and event $E\subseteq\Omega$ such that Necessitation is valid on $M$, Plausibility, \AUI{} and \KUI{} are valid in $\alpha$, and $\alpha\notin a(E)$.
\end{thm}

\begin{proof}
Let $\Omega=\{\alpha,\omega_1,\omega_2\}$. Define a binary (accessibility) relation $R$ on $\Omega$ as follows:
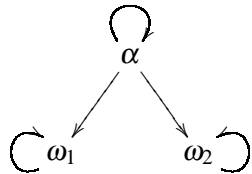
\begin{center}

\addvspace{2em}

$\xymatrix@C=3.5mm{&\alpha\ar@(ul,ur)\ar[dl]\ar[dr]&\\\omega_1\ar@(dl,ul)&&\omega_2\ar@(dr,ur)}$

\addvspace{1em}

\end{center}
$R$ induces a possibility correspondence $P$ such that $P(\sigma)=\{\tau:R\sigma\tau\}$. With $P$, define $k$ and $a$ such that for all $F\subseteq\Omega$:
\begin{itemize}
\item[] $k(F)=\{\sigma\in\Omega: P(\sigma)\subseteq F\}$
\item[] $a(F)=\begin{cases}
\{\omega_2\} & \text{if } \omega_1\in F\text{ and } \omega_2\notin F\\
\Omega & \text{otherwise}
\end{cases}$
\end{itemize}
It is routine to check that $M=\langle\Omega,k,a\rangle$, $\alpha$ and $E=\{\alpha,\omega_1\}$ witness the claim to be proven.\footnote{As suggested in \cite{modicarustichini1999uapis}, one might consider an extension of Plausibility along the following lines: For each natural number $n$, let $n$-Plausibility be $Up\to(\neg K)^np$, where $(\neg K)^n\varphi$ is defined inductively by the two clauses $(\neg K)^1\varphi=\neg K\varphi$ and $(\neg K)^{n+1}\varphi=\neg K(\neg K)^n\varphi$. For each natural number $n$, $n$-Plausibility is valid in $\alpha$.}
\end{proof}

This shows that DLR's Theorem~1(i) cannot be extended to standard state space models in which DLR's three axioms are only required to be valid in real states. In fact, the model used in the above proof of Theorem~\ref{thm-counterexample} can also be used to show that DLR's other two results cannot be extended either. For 1(ii), note that Monotonicity is valid on $M$, and that $\alpha\in k(\Omega)$ although $\alpha\notin a(E)$. For 2, note first that Weak Necessitation is valid on $M$, and that $\alpha\in a(\Omega)$ but as before $\alpha\in k(\Omega)$ and $\alpha\notin a(E)$. More generally, any state in any model which satisfies both Necessitation and Monotonicity, in addition to DLR's three axioms, will be a counterexample not just to extensions of DLR's Theorem~1(i), but also to extensions of their Theorems 1(ii) and 2.

We conclude that none of DLR's three triviality results show that standard state space models preclude unawareness. One might wonder whether plausible strengthenings of the axioms on knowledge and unawareness allow us to reinstate the triviality results. In the full paper, we argue first that this cannot be achieved by strengthening their axioms governing knowledge, and, second, that it cannot be achieved by a particular strengthening of the axioms governing unawareness. The theorems and proofs may be found in Appendix~\ref{supporting-section-2}.

\section{Three Kinds of Awareness}

\subsection{A New Triviality Result}

DLR's result had limited implications for state space models because it depended on the validity of their axioms at all states. Is there a triviality result which only uses the validity of axioms on awareness in real states, rather than their validity in all states? In fact, as we now show, widely accepted axioms on awareness \emph{do} lead to triviality even if they are imposed only at real states. The result uses the following two axioms:
\begin{itemize}
\item[] AS: $A\neg p\to Ap$
\item[] AC: $A(p\wedge q)\to(Ap\wedge Aq)$
\end{itemize}
Awareness is widely assumed to satisfy both of these axioms; see, e.g., \cite[pp.~274--275]{modicarustichini1999uapis}, \cite[p.~331]{halpern2001asfu} (axioms A1 and A2) and \cite[p.~309]{heifetzetal2008acmfiu} (axioms 1 and 2).

As the next theorem shows, these axioms lead straightforwardly to triviality.

\begin{thm}
Let $M=\langle\Omega,k,a\rangle$ be a model and $\alpha\in\Omega$ such that AS and AC are valid in $\alpha$. Then $Ap\to Aq$ is valid in $\alpha$.
\end{thm}

\begin{proof}
Consider any events $E$ and $F$, and assume $\alpha\in a(E)$. Since $E=E\cap\Omega$, $\alpha\in a(E\cap\Omega)$, and so by AC, $\alpha\in a(\Omega)$. By AS, $\alpha\in a(\emptyset)$. Since $\emptyset=\emptyset\cap F$, $\alpha\in a(\emptyset\cap F)$, and so by AC, $\alpha\in a(F)$.
\end{proof}

The crucial difference between this and DLR's triviality result is that AS and AC are only assumed to be valid in a distinguished state, for which it is shown that non-trivial unawareness in it is ruled out. 

But does awareness really satisfy both AS and AC? In the following, we will focus in particular on AC, arguing that for some important notions of unawareness, being aware of a conjunction does not entail being aware of its conjuncts.

\subsection{Attending vs. Conceiving vs. Processing}

In the literature on awareness, it is uncontentious that there is no single attitude of awareness; what is expressed by ``aware'' is a loose cluster of notions. This was noted at the very start of the literature, as witnessed by the lengthy discussions in \cite{faginhalpern1988baalr}; another detailed discussion can be found in \cite{schipperforthcominga}. We will argue that that at least some important notions of awareness do not satisfy AC (for others, as we will see, the situation is more complex).. In order to do so, we roughly distinguish the following three ways of understanding a claim of the form ``The agent is aware of \dots{}'':
\begin{itemize}
\item[(i)] The agent is attending to \dots{}
\item[(ii)] The agent has the conceptual resources required to conceive of \dots{}
\item[(iii)] The agent is able to process \dots{}
\end{itemize}
We will introduce these notions -- and distinguish between them -- using various examples found in the literature.

Consider first attention. An influential example, which first appeared in \cite{geanakoplos1989gtwpaatsac}, and which is discussed at length by DLR and numerous places in the subsequent literature, is based on the following quote from one of Arthur Conan Doyle's Sherlock Holmes stories \cite{doyle1901taosb}:

\begin{quote}
``{`Is there any other point to which you would wish to draw my attention?'\\
`To the incident of the dog in the night-time.'\\
`The dog did nothing in the night-time.'\\
`That was the curious incident' remarked Sherlock Holmes.}''
\end{quote}

Holmes's interlocutor is Inspector Gregory, a Scotland Yard detective. Before Holmes pointed out to Gregory that the dog did nothing in the night-time, Gregory was \emph{unaware} of the dog doing nothing in the night-time. Gregory's state of unawareness is naturally understood as one of \emph{inattention} -- Holmes makes Gregory aware of the dog doing nothing in the night time in the sense of bringing this fact to his attention.

Gregory's failing to attend to the dog doing nothing in the night-time must be sharply distinguished from Gregory's not being able to conceive of the the dog doing nothing in the night-time. Before Holmes alerted Gregory to the dog doing nothing in the night-time, Gregory possessed the concepts required to entertain thoughts about the dog doing nothing in the night time. Contrast this with the following example for unawareness from \cite[p.~40]{faginhalpern1988baalr}:
\begin{quote}
``How can someone say that he knows or doesn't know about $p$ if $p$ is a concept he is completely unaware of? One can imagine the puzzled frown of a Bantu tribesman's face when asked if he knows that personal computer prices are going down!''
\end{quote}
The relevant state of unawareness in this example is not merely a matter of the agent failing to attend to the relevant event or subject matter. For example, if one is unaware in the sense of being unable to conceive of an event, it must be that one does not understand the words for those notions in any language. Contrast this with the case of Inspector Gregory. Gregory understands what Holmes says: he can conceive of the dog's doing nothing. But the purported example of inconceivability does not have this structure: the tribesman is supposed to be unable to think about computers using any of his conceptual resources, no matter what he attends to. The two notions of awareness -- attending to versus being able to conceive of -- are therefore clearly distinct.%

The third notion of unawareness we want to single out is one which \cite{faginhalpern1988baalr} (see also \cite{halpern1994algorithmic} and \cite{halpernpucella2011dwlo}) focus on; it can be understood as an attempt to deal with what is known as the ``problem of logical omniscience'' in epistemic logic. In standard state spaces, if two sentences $\varphi$ and $\psi$ are equivalent in classical propositional logic, then $K\varphi$ and $K\psi$ will be true in the same states. In particular, if $K(p\vee\neg p)$ is true in a given state, then so is $K\tau$ for any propositional tautology $\tau$.
One of Fagin and Halpern's reasons for developing a logic of awareness is to obtain logics which do not have this property. They write:
\begin{quote}
``The notion of awareness we use in this approach is open to a number of interpretations. One of them is that an agent is aware of a formula if he can compute whether or not it is true in a given situation within a certain time or space bound. This interpretation of awareness gives us a way of capturing resource-bounded reasoning in our model.''
\end{quote}

Being unaware of $\varphi$ in the sense of not being able to process $\varphi$ is clearly distinct from failing to attend to $\varphi$: although Gregory did not attend to the dog doing nothing in the night-time, he had no difficulties processing the claim that the dog did nothing in the night-time. Not being able to process $\varphi$ is also clearly distinct from not being able to conceive of $\varphi$: Gregory might not have been able to process an extremely complicated propositional tautology using only negation, conjunction and the sentence ``the dog did nothing in the night-time'', but he clearly possessed all the concepts required to entertain it.

\subsection{Awareness of Conjunctions}

Let's return to the new triviality result introduced at the start of this section. As already advertised, we believe that the principle AC, which says that an agent who is aware of a conjunction is aware of its conjuncts, may be plausible for one notion of awareness, but it is not for the other two. 

Consider first awareness as the ability to process. This is plausibly a relation of agents to \emph{sentences}, as part of what it takes to process \dots{} is to be able to find out what the sentence ``\dots{}'' means. AC may hold if awareness is understood as the ability to process: It is natural to assume that an agent who is able to process a conjunction ``... and ---'' is also able to process ``...'' and ``---''. As noted already in \cite[p.~54]{faginhalpern1988baalr}, even this may fail: an agent might be able to recognize that a very long sentence has the form $\varphi \wedge \neg \varphi$, and so be able to process it, although she is unable to process the complex $\varphi$ on its own. Resolving this controversy may require distinguishing further among different notions of processing, and the appropriate resolution may depend on the intended application.

The relations of attention and conceivability are different from the ability to process. In particular, they are plausibly relations agents have not to sentences but to what sentences express. We might call these entities \emph{contents} or \emph{propositions}, but in keeping with the terminology employed above, we call them \emph{events}. In the full paper, we discuss the difference between sentences and events in more detail, and motivate the assumption of a \emph{coarse-grained theory of events}, according to which events form a complete atomic Boolean algebra. A consequence of this assumption -- which is implicit in all standard state-space based modeling techniques -- is that sentences which are equivalent in propositional logic have identical contents. 

With this understanding of attention and conceivability as relations of agents to coarse-grained events, consider first attention. Assume that after the conversation with Holmes quoted above, Gregory is alone thinking about the case, and attending to the event of the dog barking in the night ($p$). He is not, however, attending to event of Holmes at that moment smoking a pipe ($q$). It is then natural to say that Gregory is also not attending to the conditional event that \emph{if} Holmes is currently smoking a pipe \emph{then} the dog barked in the night ($q \to p$). But notice that according to the coarse-grained Boolean theory of events, the event that the dog barked in the night ($p$) is identical to the event that if Holmes is smoking a pipe then the dog barked in the night, and the dog barked in the night ($(q \to p )\land p$). So if AC were valid, then since Gregory is attending to the event of the dog barking in the night, he would be attending to the event that if Holmes is smoking a pipe then the dog barked in the night. But by assumption Gregory is not attending to this last event. Thus it follows from the coarse-grained theory of events that AC must be rejected.

A similar example can be given if we understand awareness as conceivability. Assume our agent does not have the conceptual resources to entertain the event of there being a black hole. According to the assumed coarse-grained theory of content, the event of there being a black hole and there being no black hole is identical to any event expressed by a propositional contradiction, such as the event of there being a sheep and there being no sheep. The agent might well have the conceptual resources to entertain the event of there being a sheep and there being no sheep, without having the conceptual resources to entertain the event of there being a black hole.

If we adopt, as usual, a theory of events which identifies the event expressed by sentences which are equivalent in propositional logic, AC appears to be inappropriate. Thus the new triviality result with which we started this section also does not establish that standard state space models preclude unawareness understood as inattention or inability to conceive.

\section{Partitional Models}

So far, we have shown that standard state-space models escape certain putative impossibility results for models of attention and conceivability. But this does not establish that standard space models can provide fertile models of these notions. In the remainder of the paper, we define, motivate, and examine a class of standard state space models for representing attention and the ability to conceive. 

To show that our models generalize smoothly to the multi-agent case, from now on we use a language $L_I$ parametrized to an arbitrary set of agent-indices $I$ which is defined as the language $L$ above, except that the operators $A_i$ and $K_i$ are indexed to $i\in I$. Models are consequently tuples of the form $\langle\Omega,k^i,a^i\rangle_{i\in I}$.

The models we will be working with are defined as follows:

\begin{definition}
$\langle\Omega,R^i,\approx^i\rangle_{i\in I}$ is a \emph{partitional model} if $\Omega$ is a set and for each $i\in I$, $R^i$ is a binary relation on $\Omega$ which is reflexive and transitive, and $\approx^i$ a function which maps each $\omega\in\Omega$ to an equivalence relation $\approx^i_\omega$ on $\Omega$.
\end{definition}

Here and in what follows, we make use of the fact that each equivalence relation corresponds to a unique partition, and \emph{vice versa}; accordingly, we treat them as interchangeable.

Partitional models can be used to generate standard models in the following way: $R^i$ specifies states of knowledge just as in Theorem~\ref{thm-counterexample}. The idea behind $\approx^i$ is that the events the agent is aware of at $\omega$ are the events which are unions of sets of equivalence classes of $\approx^i_\omega$ (equivalently: unions of sets of cells of the induced partition). So for each $i\in I$, let $R^i$ and $\approx^i$ determine functions $k^i$ and $a^i$ on events on $\Omega$ as follows:
\begin{description}
\item[$k^i(E)=$] $\{\sigma\in\Omega: P^i(\sigma)\subseteq E\}$, where $P^i(\sigma)=\{\tau:R^i\sigma\tau\}$
\item[$a^i(E)=$] $\{\sigma\in\Omega:$ for all $\rho$ and $\tau$ such that $\rho\approx^i_\sigma\tau$, $\rho\in E$ iff $\tau\in E\}$
\end{description}
Let the standard model determined by a partitional model $\langle\Omega,R^i,\approx^i\rangle_{i\in I}$ be $\langle\Omega,k^i,a^i\rangle_{i\in I}$, with $k^i$ and $a^i$ as just defined. On such a standard model, $L_I$ can be interpreted as above; obviously, this induces a way of interpreting $L_I$ directly on partitional models.

\subsection{The Attitude of Attention}\label{sect-att-qda}

In order to motivate partitional models as models of limited attention, we suggest that attention in the sense we have been using the term should primarily be understood as an attitude towards questions. There are many available formal approaches to modeling questions (for an overview, see \cite{krifka2011questions}). For concreteness, we'll adopt a standard approach, representing questions as partitions of the state space (see  \cite{groenendijkstokhof1984sotsoqatpoa}, building on \cite{hamblin1973qime} and \cite{karttunen1977sasoq}). Although we think the attitude to questions is primary, we will follow the literature on awareness, in axiomatizing a notion of attention which has events as its objects.  The relationship between this attitude to events and the attitude toward questions will be as follows: an agent attends to the question $Q$ if and only if the agent attends to every partial answer to $Q$. Using partitions to model questions, partial answers are unions of sets of cells, corresponding to how standard models are derived from partitional models.\footnote{Section \ref{decisiontheory} discusses related models developed in \cite{kets2014bounded,kets2014finite}.}

\subsection{Partitions for Conceivability}

To motivate the use of partitional models of conceivability, assume that the agents to be modeled have the concept of negation and (infinitary) conjunction, so that the set of events they can conceive of are closed under complement and arbitrary intersection. This is mathematically equivalent to requiring that this set is derived from an equivalence relation as above.

\subsection{An Example}\label{sect-partitional-model-example}

It will be useful to have a concrete partitional model before us, as a running example. The following model shows that there are non-trivial partitional models; for simplicity, a single-agent case is specified. Let $M=\langle\Omega,R,\approx\rangle$, with $\omega R\nu$ iff $\omega=1$ or $\omega=\nu$, and $\approx$ given by the following equivalence classes:

\

\begin{itemize}
\item[$\approx_1:$] $\{1\},\{2,3,4\}$
\item[$\approx_2:$] $\{1\},\{2\},\{3,4\}$
\item[$\approx_3:$] $\{1\},\{3\},\{2,4\}$
\item[$\approx_4:$] $\{1\},\{4\},\{2,3\}$
\end{itemize}
Thus, at 1, the strongest event known by the agent is $\Omega$, and at each other state $n$, it is $\{n\}$. At each state $n$, the events the agent is aware of are the events which don't distinguish between any states in $\Omega\backslash\{1,n\}$.

Drawing the four states in a circle, starting with 1 at the top and going clockwise, we can draw each equivalence relation in a similar smaller circle, connecting two states by a sequence of lines if they are related by the relevant equivalence relation:
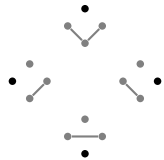
\begin{center}
\begin{e-p-sys}{4}
\draw [line] (1-2) -- (1-3) -- (1-4);
\draw [line] (2-3) -- (2-4);
\draw [line] (3-2) -- (3-4);
\draw [line] (4-2) -- (4-3);
\end{e-p-sys}
\end{center}
This is a partitional model in which there is non-trivial unawareness at each state. We will appeal to it below in order to show the consistency of various constraints.

\subsection{Axioms}

Given a class of models $\mathtt C$, a set of sentences $\Sigma \subseteq L_I$ is the logic of $\mathtt C$ if and only if $\Sigma$ contains exactly those sentences which are valid on $\mathtt C$. Characterizing the logic of a class of models gives us a formal perspective from which to assess what assumptions our models encode about agents knowledge and awareness.

Thus we may ask: What is the logic of partitional models? Standard techniques on completeness results in modal logic are easily adapted to obtain the following result.

\begin{thm}
A formula is valid on all partitional models if and only if it is derivable in the calculus with the following axiom schemas and rules:
\begin{itemize}
\item[] PL: Any substitution instance of a theorem of propositional logic.
\item[] K-K: $K_i (\varphi \to \psi) \to (K_i \varphi \to K_i \psi)$
\item[] K-T: $K_i \varphi \to \varphi$
\item[] K-4: $K_i \varphi \to K_i K_i \varphi$
\item[] A-Neg: $A_i \varphi \to A_i\neg \varphi$
\item[] A-M: $(A_i \varphi \wedge A_i \psi)\to A_i( \varphi \wedge \psi)$
\item[] A-N: $A_i \top$
\item[] MP: From $\vdash \varphi$ and $\vdash \varphi\to\psi$ infer $\vdash \psi$
\item[] K-RN: From $\vdash \varphi$ infer $\vdash K_i \varphi$
\item[] A-RE: From $\vdash \varphi \leftrightarrow \psi$ infer $\vdash A_i \varphi \leftrightarrow A_i \psi$
\end{itemize}
Moreover, the logic is decidable.
\end{thm}

A proof is given in Appendix B.

\subsection{DLR Once More}\label{DLR-once-more}

Consider again DLR's three axioms. Given our discussion above, it is natural to consider partitional models where DLR's axioms are required to be valid in some distinguished state:

\begin{definition}
$\langle \Omega, \alpha, R^i, \approx^i \rangle_{i\in I}$ is a \emph{partitional DLR mod\-el} if $\langle \Omega, R^i, \approx^i \rangle_{i\in I}$ is a partitional model and Plausibility, \KUI{} and \AUI{} (for each $i\in I$) are valid in $\alpha$.
\end{definition}

We now show that DLR's triviality result cannot be revived in partitional models:

\begin{thm} There is a partitional DLR model $\langle \Omega, \alpha, R^i,$ $\approx^i\rangle_{i\in I}$ and an event $E\subseteq \Omega$ such that $\alpha \in U(E)$. \end{thm}
 
 \begin{proof} Simply distinguish state 1 in the model presented in section~\ref{sect-partitional-model-example}. \end{proof}
 
We conjecture that the logic of partitional DLR models can be axiomatized as follows:

\begin{cnj}
Add the following axioms to the theorems of the axiom system in Theorem~4 and close under modus ponens:
\begin{itemize}
\item[] P: $U_i\varphi\to(\neg K_i \varphi \wedge\neg K_i\neg K_i\varphi)$
\item[] AU: $U_i\varphi\to U_i U_i\varphi$
\end{itemize}
A formula is derivable in this calculus if and only if it is valid in every distinguished state of every partitional DLR model.
\end{cnj}

Note that $\neg KU\varphi$ can be derived using P, AU and K-T.

The present result shows that we can impose the DLR axioms without trivializing partitional models. But we confess to doubts about whether these axioms are appropriate. Just as with AC, once we understand more clearly the character of attention and conceivability, as well as the distinction between sentences and what they express, DLR's axioms become much less compelling. The clearest case concerns Plausibility and attention. A consequence of the contrapositive of Plausibility is $K \varphi \to A \varphi$. But this principle is false for attention. You know that there are more than four stars in the universe, but we doubt that you were attending to the question of how many stars there are prior to reading the previous clause. As we discuss in more detail in the full paper, the coarse-grained conception of content together with clarity about the notion of awareness to be modeled cast doubt on DLR's axioms.

\subsection{Propositional Quantification}

A challenge to some approaches to unawareness is to represent propositionally quantified statements. E.g., earlier models by Halpern made the claim that the agent knew she was unaware of something unsatisfiable (cf. \cite{halpernrego2009rakou} and \cite{halpernrego2013rakour}). In standard state space models such as ours, it is trivial to add propositional quantifiers without any such consequences. To do so, we write $v[E/p]$ for the valuation function which maps $p$ to $E$ and every other proposition letter $q$ to $v(q)$:\footnote{See already \cite{kripke1959actiml}, and \cite{fine1970pqiml} for a more systematic development. See \cite{fritzunpublishedlfpc} for results on the complexity of propositional quantifiers in the related setting of \cite{fritzunpublishedpc}.}
\begin{itemize}
\item[] $\llbracket\forall p\varphi\rrbracket_{M,v}=\bigcap_{E\subseteq\Omega}\llbracket\varphi\rrbracket_{M,v[E/p]}$
\end{itemize}
To illustrate that these quantifiers behave just as one would expect, note that in state 1 of the example described in section~\ref{sect-partitional-model-example}, the agent knows that she is unaware of something without there being something that she knows to be unaware of: $K\exists pUp\wedge\neg\exists pKUp$ is true in this state.

\subsection{Closure and Automorphisms}

%In this section, we discuss some further elaborations on partitional models, and show that they too are consistent with the DLR axioms. 

In partitional models, what agents are aware of (attend to/can conceive) is closed under negation and conjunction. One might wonder whether we can also impose the constraints that what agents are aware of must be closed under awareness and knowledge. In other words, whether there are models on which the following axioms are valid:
\begin{itemize}
\item[A-4ij] $A_ip\to A_iA_j p$
\item[AK-4] $A_ip\to A_iK_j p$
\end{itemize}

%We can think of these principles as appropriate for representing agents who are attending to questions formed by applying knowledge or awareness to questions they are attending to. Or in the case of conceivability, they represent agents who have the concepts of knowledge and unawareness.

To provide models which validate these principles we adapt the coherence constraint of \cite{fritzunpublishedpc}.\footnote{The following notion of coherence differs importantly from that of \cite{fritzunpublishedpc} in that ${\approx}^i_x$ here need not relate $x$ only to $x$.} The idea behind it is most easily described for awareness as conceivability, taking the equivalence relations of partitional models to represent a relation of indistinguishability using conceptual resources available to the relevant agent at the relevant state. Coherence requires that if two states are indistinguishable in this way, then there must be a way of permuting the state space in a way which preserves all structural facts about knowledge and awareness, as well as all the events which the relevant agent is aware of at the relevant state.

Let $M=\langle\Omega,R^i,\approx^i\rangle_{i\in I}$ be a partitional model. A permutation $f$ of $\Omega$ is an automorphism of $M$ if for all $i\in I$,
\begin{itemize}
\item[(i)] for all $x,y,z\in\Omega$, $y\approx^i_xz$ iff $f(y)\approx^i_{f(x)}f(z)$, and
\item[(ii)] for all $x,y\in\Omega$, $R^ixy$ iff $R^if(x)f(y)$.
\end{itemize}
A state $x\in\Omega$ \emph{coheres} if for all $i\in I$ and $y,z\in\Omega$ such that $y\approx^i_xz$ there is an automorphism $f$ of $M$ such that $f(y)=z$ and $f\subseteq{\approx^i_x}$ (i.e., $\omega\approx^i_x\omega$ for all $\omega\in\Omega$).
It's routine to verify that A-4ij and AK-4 are valid in any coherent state of a partitional model.

Once again, the model presented in section~\ref{sect-partitional-model-example} demonstrates the satisfiability of this constraint: every state in this model is coherent. Since the model also satisfies the DLR axioms at state $1$, it shows that even if we were to uphold the DLR axioms, imposing them together with coherence would not trivialize state space models of awareness.

\section{Decision Theory}\label{decisiontheory}

In section~\ref{DLR-once-more} the example of the number of stars illustrated how one may believe and know things to which one is not attending; clearly this kind of inattention may also affect choice-behavior. One advantage of standard state spaces is that we can use the usual decision-theoretic framework to represent the effects of inattention on choice-behavior.\footnote{We do not here attempt to back-form what the agent is aware of from her choice-dispositions, as \cite{morris1996tlobabc,morris1997adok} do for belief, and \cite{schipper2013adseu,schipper2014pbu} do for awareness.}

The usual decision theoretic representation of an agent's beliefs uses a measure-space $\langle S, \mathcal B, \mu \rangle$.\footnote{This can be derived in any of the standard ways: e.g. \cite{savage1954fos}, \cite{vonneumannmorgenstern1944tgeb}, \cite{anscombeaumann1963dosp}, \cite{bolker1967saouasp}, \cite{jeffrey1983lod}, \cite{broome1990bjeut}.}  To generate a partitional model, we enrich this description of the agent by selecting $\mathcal B^C$, a complete atomic subalgebra of $\mathcal B$, to generate a representation of what the agent attends to in the context of choice: $\langle S, \mathcal B, \mu, \mathcal B^C\rangle$. The atoms of $\mathcal B^C$ are a partition of $S$, so this structure gives rise to a partitional model of unawareness. The distribution the agent ``uses'' in a choice context is given by letting $\mu^C (E) = \mu(E)$ for all $E \in \mathcal B^C$ and undefined otherwise. The events the agent ``explicitly believes'' in the context can then be defined as the events of which the agent is certain in $\mu^C$. The algebra $\mathcal B^C$ can also be used to parametrize ``expanding'' and more generally ``changing'' awareness, represented as transitions between different complete atomic sub-algebras of $\mathcal B$.\footnote{See the full paper for an alternative related to that of  \cite{karniviero2013rb}.} Since different algebras will determine different \emph{explicit} beliefs in different contexts, this changing awareness can also represent effects of limited attention such as framing effects or failures of recall.

An approach along these lines has already proven fruitful in epistemic game theory. \cite{kets2014bounded} and \cite{kets2014finite} develop Harsanyi type spaces in which players' beliefs may be defined on different $\sigma$-algebras. If the algebras are taken as the events the agent is attending to, one may interpret these models as examples of agents who fail to attend to questions about the higher-order beliefs of others, and thus do not have \emph{explicit} beliefs over events which can be defined only by the level-$n$ beliefs of others for large enough $n$.

\subsection{Speculative Trade}

An important test of approaches to unawareness has been how they fare with speculative trade (\cite{heifetzetal2013ubast}). Building on the work of \cite{AumannAgreeing}, \cite{MilgromStokey} proved their famous ``no-trade'' theorem, illustrating the extreme strength of $S5$ knowledge together with a common prior.  One aim of representing ``bounded'' agents such as those with limited attention is to escape such paradoxes (for this perspective, see \cite{Morris1995}, \cite{ledermanpeople}). Accordingly, we now provide a partitional DLR model with a common prior in which speculative trade is possible.

As is well known (see \cite{geanakoplos1989gtwpaatsac}, \cite{samet1990ignoring}, \cite{rubinsteinwolinsky1990otloatdtr}), the ``no-trade'' theorem does not hold in general if agents' accessibility relations $R^i$ are merely transitive and reflexive, but are not required to form an equivalence relation. Plausibility is incompatible with the $R^i$ forming an equivalence relation (\cite{modicarustichini1994aapis}). Still, the DLR axioms together with partitional awareness models impose substantial further constraints, which might be thought to rule out speculative trade. We now construct a partitional DLR model to show that speculative trade can still occur in the presence of DLR's axioms and nontrivial unawareness.
%
%\begin{center}
%\begin{e-p-sys}{5}
%\draw [line] (1-2)--(1-3);
%\draw [line] (2-2)--(2-3);
%\draw [line] (4-4)--(4-3);
%\draw [line] (5-4)--(5-3);
%\end{e-p-sys}
%\end{center}

Let the states be $W=\{1,2,3,4,5\}$ and the agents be Alice, $A$, and Bob $B$. The accessibility relations are defined so that: $1R^Ax$ iff $x\leqslant 3$; $5R^Ax$ iff $x \geqslant 3$ and otherwise $wR^Ax$ iff $w=x$, while $R^B=W \times W$. The partitions of the agents are induced by $\approx^A_1=\approx^A_2=\{\{1\}, \{2,3\}, \{4\}, \{5\}\}$; $\approx^A_4=\approx^A_5=\{\{1\}, \{2\},\{3, 4\},$ $ \{5\}\}$; and for all $w$, $\approx^B_w=\approx^A_3=\{\{1\}, \{2\},\{3\},\{ 4\}, \{5\}\}$. The agents' common prior is the uniform one, and two acts $f$ and $g$ have utility as follows: $f(1)=f(5)=1$, $f(2)=f(4)=5$, $f(3)=7$; $g(w)=4$ for all $w \in W$. If the agents update by conditionalization on their implicit knowledge, then Alice will invariably maximize utility by choosing $f$ (since in states $2, 3, 4$ she is certain it does better, and in states $1$ and $5$ she expects to gain $1/3\cdot1+1/3\cdot5 + 1/3 \cdot 7>4$). Bob meanwhile does not update at all, so that he strictly prefers $g$ (since $4>2/5\cdot 1 + 2/5 \cdot 5 + 1/5\cdot 7$).

\section{Conclusion}

Standard state space models of attention and conceivability are at least as successful as current non-standard state space models. The non-standard models are, however, more complicated, and it is unclear that this complexity affords any advantages in predictive strength or accuracy. Standard state space models of these phenomena promise to lead to a rich and rewarding theory, posing technical and conceptual challenges, and offering connections to related work by linguists, philosophers and logicians -- as well as work on bounded reasoning elsewhere in economic theory.

\section{Acknowledgments}

Thanks to Eddie Dekel, Bart Lipman and Tim Williamson for very helpful discussions of drafts of this paper. Peter Fritz was supported by an AHRC/Scatcherd European Scholarship and the German Academic Exchange Service; Harvey Lederman acknowledges the support of the Clarendon Fund.

\bibliographystyle{eptcs}
\bibliography{awarenessbib}

\begin{appendices}

\section{Material Supporting Section 2}\label{supporting-section-2}

\subsection{Stronger Assumptions about Knowledge}

The model used in Theorem~\ref{thm-counterexample} already validates a number of attractive axioms on knowledge, suggesting that strengthening DLR's constraints on knowledge is unlikely to yield an interesting triviality result. In particular, the following axioms are valid on the model:
\begin{itemize}
\item[] Distribution: $(Kp\wedge Kq)\to K(p\wedge q)$
\item[] Anti-Necessitation: $\neg K\bot$
\item[] Reflexivity: $Kp\to p$
\item[] Positive Introspection: $Kp\to KKp$
\end{itemize}

We can show more systematically that any strengthening of the axioms of knowledge which rules out unawareness does so trivially in the same way as Negative Introspection does. On the very mild assumption that the agent doesn't know the contradiction $\bot$, we can characterize the conditions under which a given model for the knowledge of an agent can be extended to an unawareness model in which the agent is unaware of a given $p$ at a given point $\alpha$ in which DLR's three axioms are valid. Let a model $M'$ \emph{extend} a \emph{knowledge model} $M=\langle\Omega,k\rangle$ just in case $M'=\langle\Omega,k,a\rangle$ for some function $a:2^\Omega\to 2^\Omega$.

\begin{thm}\label{thm-extension}
Let $M=\langle\Omega,k\rangle$ be a knowledge model, $\alpha\in\Omega$ and $E\subseteq\Omega$ such that Anti-Necessitation is valid in $\alpha$. $M$ has an extension such that Plausibility, \KUI{} and \AUI{} are valid in $\alpha$ and $\alpha\notin a(E)$ if and only if

(i) $\alpha\in-k(E)\cap-k-k(E)$, and

(ii) there is an event $F$ such that $a\in F\cap-k(F)\cap-k-k(F)$.\footnote{This result also holds if we replace Plausibility by $n$-Plausibility, for all natural numbers $n$, and (i) and (ii) by the correspondingly iterated conditions.}
\end{thm}

\begin{proof}
Assume first that (i) and (ii). Let $a:2^\Omega\to 2^\Omega$ be defined so that $a(E)=a(F)=-F$, and $a(G)=\Omega$ for all other events $G$, and consider the model $M'=\langle\Omega,k,a\rangle$. It is routine to verify that Plausibility, \KUI{} and \AUI{} are valid in $\alpha$, and $\alpha\notin a(E)$, appealing to Anti-Necessitation in the proof for \KUI{}.

For the converse, note that (i) follows the validity of Plausibility in $\alpha$. For (ii), let $F=-a(E)$. Then $\alpha\in F$, by \AUI{}, $\alpha\notin a(F)$, and so by Plausibility, $\alpha\in-k(F)\cap-k-k(F)$. 
\end{proof}

In particular, as long as the constraints on knowledge allow for there to be an event $E$ and a state $\alpha\in E$ such that $\alpha\in-k(E)\cap-k-k(E)$, standard state space models and DLR's three axioms will not preclude non-trivial unawareness.

\subsection{Stronger Assumptions about Awareness}

These results demonstrate that no plausible strengthening of the axioms governing knowledge will re-instate triviality. But what if we strengthen the axioms on awareness themselves?

To investigate this issue formally, extend the language $L$ by a unary operator $\CK$ for common knowledge. To define its interpretation on a model $M=\langle\Omega,k,a\rangle$, derive the following functions on events: $k^1(E)=k(E)$, $k^{n+1}(E)=kk^n(E)$, and $\ck(E)=\bigcap_{1\leq n<\omega}$ $k^n(E)$.
\begin{itemize}
\item[] $\llbracket\CK\varphi\rrbracket_{M,v}=\ck(\llbracket\varphi\rrbracket_{M,v})$
\end{itemize}
With this, we consider the following additional axioms on awareness:
\begin{itemize}
\item[] $\CK$-Plausibility: $Ap\to\CK(Up\to(\neg Kp\wedge\neg K\neg Kp))$
\item[] $\CK$-\KUI{}: $Ap\to\CK(\neg KUp)$
\item[] $\CK$-\AUI{}: $Ap\to\CK(Up\to UUp)$
\end{itemize}

These additional axioms are also compatible with non-trivial unawareness. In fact, they are valid in state $\alpha$ of the example in the proof of Theorem~\ref{thm-counterexample}. More generally, Theorem~\ref{thm-extension} can be extended straightforwardly to these three additional axioms, given the weak assumption that Necessitation and Anti-Necessitation are valid:

\begin{thm}\label{thm-CK-extension}
Let $M=\langle\Omega,k\rangle$ be a knowledge model in which Necessitation and Anti-Necessitation are valid, $\alpha\in\Omega$ and $E\subseteq\Omega$. $M$ has an extension such that Plausibility, \KUI{}, \AUI{}, $\CK$-Plausibility, $\CK$-\KUI{} and $\CK$-\AUI{} are valid in $\alpha$ and $\alpha\notin a(E)$ if and only if

(i) $\alpha\in-k(E)\cap-k-k(E)$, and

(ii) there is an event $F$ such that $\alpha\in F\cap-k(F)\cap-k-k(F)$.\footnote{Again, we can extend this result to $n$-Plausibility for all $n$ analogously to the extension in the previous footnote.}
\end{thm}

\begin{proof}
We establish (i) and (ii) as in the proof of Theorem~\ref{thm-extension}. Assuming (i) and (ii), we define $a$ as in the proof of Theorem~\ref{thm-extension}, where it is noted that Plausibility, \KUI{} and \AUI{} are valid in $\alpha$, and $\alpha\notin a(E)$. For the $\CK$-conditions, consider any event $G$ such that $\alpha\in a(G)$. Then by construction of $a$, $a(G)=\Omega$. Therefore $a(G)\cup H=\Omega$ for any event $H$, and so $\alpha\in\ck(a(G)\cup H)$ by Necessitation, which establishes the validity of $\CK$-Plausibility and $\CK$-\AUI{} in $\alpha$. For $\CK$-\KUI{}, note that by Anti-Necessitation, $k-a(G)=\emptyset$, so $-k-a(G)=\Omega$, from which $\alpha\in\ck(-k-a(G))$ follows again by Necessitation.
\end{proof}

\section{Proof of Theorem 4}

\begin{proof}
Since the formulas derivable in this calculus form a classical model logic in the sense of \cite{segerberg1971aeicml}, we can apply the standard canonical model construction technique; in particular, consider the smallest canonical model (see \cite[chapter~9]{chellas1980ml}, especially p.~254). Consider any formula $\varphi$ not provable in the above calculus, and let $\Gamma$ be the set of subformulas of $\varphi$ closed under Boolean combinations. A standard filtration of the canonical model through $\Gamma$ produces a finite model in which $\varphi$ is false. It is routine to prove that the neighborhood function for $A_i$ associates with each state a field of sets; since the model is finite this field is generated by an equivalence relation, as required. The above filtration can be chosen in such a way as to preserve the transitivity of the relation for $K_i$; reflexivity is preserved by any filtration (see, e.g. \cite[chapter~3]{chellas1980ml}, especially p.~106, or \cite{BlueBook} p.~80).

The above argument also establishes that the logic thus axiomatized has the finite model property and so is decidable.
\end{proof}

\end{appendices}

\end{document}